\documentclass{article}

     \usepackage[final,nonatbib]{amlc_2022}

\usepackage[utf8]{inputenc} 
\usepackage[T1]{fontenc}    
\usepackage{hyperref}       
\usepackage{url}            
\usepackage{booktabs}       
\usepackage{amsfonts}       
\usepackage{nicefrac}       
\usepackage{microtype}      
\usepackage{xcolor}         
\usepackage[ruled,vlined]{algorithm2e}
\usepackage{amsmath}
\usepackage{amsthm}

\newtheorem{theorem}{Theorem}

\newtheorem{assumption}{Assumption}
\newtheorem{definition}{Definition}
\newtheorem{remark}{Remark}

\def\argmax{\mathop{\mathrm{argmax} }}
\def\A{\mathcal{A}}

\def\E{\mathbb{E}}
\def\F{\mathcal{F}}

\def\I{\mathcal{I}}
\def\M{\mathcal{M}}
\def\P{\mathbb{P}}

\def\S{\mathcal{S}}
\def\T{\mathcal{T}}

\title{Contextual Bandits for Evaluating and Improving Inventory Control Policies}

\author{%
  Dean Foster\\
  Amazon\\
  \texttt{foster@amazon.com} \\
   \And
   Randy Jia \\
   Amazon\\
   \texttt{randyjia@amazon.com} \\
   \AND
   Dhruv Madeka \\
   Amazon \\
   \texttt{maded@amazon.com} \\
}

\begin{document}

\maketitle

\begin{abstract}
Solutions to address the periodic review inventory control problem with nonstationary random demand, lost sales, and stochastic vendor lead times typically involve making strong assumptions on the dynamics for either approximation or simulation, and applying methods such as optimization, dynamic programming, or reinforcement learning. Therefore, it is important to analyze and evaluate any inventory control policy, in particular to see if there is room for improvement. We introduce the concept of an equilibrium policy, a desirable property of a policy that intuitively means that, in hindsight, changing only a small fraction of actions does not result in materially more reward. We provide a light-weight contextual bandit-based algorithm to evaluate and occasionally tweak policies, and show that this method achieves favorable guarantees, both theoretically and in empirical studies. 
\end{abstract}

\section{Introduction}
A fundamental yet difficult problem in operations research is the periodic review inventory control problem, in which a decision maker must periodically execute an inventory replenishment strategy to satisfy demand while simultaneously minimizing the cost of holding excess inventory. We consider the inventory control problem with random demand, lost sales, and stochastic vendor lead times. In this setting, demand is assumed to be a random variable with unknown (possibly nonstationary) dynamics, and is required to be met immediately with in-stock units in the warehouse (assume unmet sales are lost). Additionally, the time between when an order is placed and when it arrives in the warehouse is also an unknown random variable, referred to as the vendor lead time. The goal of the decision maker is to determine an ordering strategy every review period (e.g. weekly), for every single product being sold.

{\bf Related Work.} Traditional operations research methods (newsvendor, dynamic programming) have typically been employed as approximate solutions to this problem, by making the problem more tractable with simplifying assumptions on the problem dynamics. However, another fundamental difficulty is the planning problem: orders placed at the present affect the inventory state conditions in the future. In order to address this, reinforcement learning (RL) has been proposed as an end-to-end solution to inventory management \cite{scotrl}. While the RL formulation of the problem explicitly optimizes the desired long run reward, RL has high sample complexity, meaning that efficient practical algorithms are generally intractable without certain simplifying assumptions on the problem dynamics, such as requiring access to an accurate simulator \cite{scotrl} or a sufficiently large amount of offline data satisfying certain distributional assumptions \cite{wang2020statistical}. The authors of \cite{scotrl} build a high-fidelity simulator and train an RL policy under the simulated environment, iterating over billions of ``years'' worth of data. However, there have been observable gaps between the simulation and reality when the policy is deployed in experimental trials. 

In this paper, we attempt to characterize the quality of an inventory control policy beyond the traditional sense of regret, which is normally defined as the difference between a policy's expected reward and the best policy's reward. Most of the time, our policies are determined by models which do not precisely match the exact, unknown dynamics of the real world, and even in hindsight it is difficult to counterfactually determine the best policy. We introduce a property known as equilibrium, and show that it is an achievable objective for evaluating and occasionally improving inventory policies in a light-weight, sample-efficient manner.

\section{Problem Statement}
\subsection{Preliminaries}
We begin by introducing the inventory control problem formulated as a fairly generic Markov Decision Process (MDP) \cite{puterman2014markov}. The inventory control problem can be summarized as follows: at each review period $t$, an agent is faced with the decision of how many units of inventory to order (or how many units to remove through a price markdown). These orders do not arrive immediately, instead, they arrive after some lead time drawn from an unknown (but bounded) vendor lead time (VLT) distribution. On-hand inventory and units arriving at time $t$ from a prior order are used to fulfill orders of demand, which is determined following some unknown demand distribution. At the end of each time period, revenues and costs are assessed. For precise definitions and dynamics, we refer readers to \cite{scotrl}. 

Now, we present some notation regarding our problem formulation. Define MDP $\M$ as a tuple $(\S, \A, r, \P, \gamma)$, where $\S$ is the state space, $\A$ the action space, $r$ the reward function, $\P$ the state transition probability matrix, and $\gamma$ the discount factor. At each time $t$, the decision maker observes the state $s_t \in \S$, and picks an action $a_t \in \A$ to play. Then, a reward is generated via $r_t := r(s_t,a_t)$, and the next state $s_{t+1}$ is generated via from the probability vector $P(s_t, a_t)$. 

We define a policy $\pi: \S \rightarrow \A$ to be a mapping from the state space to the action space, that is, for any given state, the policy defines what actions the decision maker should take. The decision maker is faced with determining the best policy to follow that maximizes $R_T$, the cumulative discounted reward in time horizon $T$: 
\[R_T := \sum_{t=1}^T \gamma^t r_t.\]

Typically, we try to minimize \emph{regret}, or the difference between the reward of our learned policy $\pi$ and the optimal policy $\pi^*$ (which is known to exist for this MDP, see \cite{puterman2014markov}):
\[Regret(T) = \sum_{t=1}^T \gamma^t r^{\pi^*}_t - \sum_{t=1}^T \gamma^t r^{\pi}_t.\]

This formulation of the inventory control problem reveals one of the fundamental challenges brought by the vendor lead times: the state space necessarily includes information of past, in-flight orders. This roughly increases the size of the state space by a factor of $|\A|^L$, where $L$ is the largest possible vendor lead time, which automatically makes for very difficult learning due to such large state spaces. This is typically one of the reasons further assumptions or simulators are required to address this issue.

\subsection{Policy evaluation}
Given that our inventory control policy requires an estimated model of the world, there are gaps between our policy as the true optimal policy under the real environment. Therefore, it is important and fair to ask, ``how good is our policy?'' The objective in this paper, therefore, is not to determine the best true policy $\pi^*$; rather, we are given a policy $\pi^{model}$ and tasked with evaluating it. Instead of evaluating whether $\pi^{model}$ is close to $\pi^*$ (i.e., regret), since $\pi^*$ cannot be easily computed counterfactually, we instead determine whether a small number of local changes in the policy yields a significant increase in total reward. If this is not possible, then we say that our policy is in equilibrium. Otherwise, there is a simple way to improve $\pi^{model}$ that leads to materially more reward. The following formal definition of equilibrium policies is derived from the game theoretic notion of equilibrium in \cite{deanequilibrium}.

\begin{definition}
\label{def:ep}
Let $\T$ be a set of times. Given starting state $s_0$ and policy $\pi$ (which induces states $s_1,s_2,..,s_T$), let $\Pi_{\T}$ be the set of policies such that for any $\pi' \in \Pi_{\T}$ (which induces states $s'_1,s'_2,...,s'_T$), $\pi(s_t) = \pi'(s'_t)$ for all times $t \not\in \T$. Then, $\pi$ is an $\epsilon$-equilibrium policy for all times $t \in \T$ if
\begin{equation} \label{eq:equil}
\frac{1}{T}E_\pi[R_T|s_0] \geq \max_{\pi' \in \Pi_{\T}} \frac{1}{T}E_{\pi'}[R_T|s_0] - \epsilon.
\end{equation}
\end{definition}

Now, rather than checking whether a policy achieves close-to-optimal reward (i.e., low regret), we instead check whether a policy is an $\epsilon$-equilibrium policy, specifically if a relatively small number of local changes (at times $t \in \T$) will improve reward significantly (more than $\epsilon$). To do so, we will utilize some techniques from the theory of contextual bandits.

\section{Contextual Bandit for policy evaluation}
Contextual bandits are a fundamental model in online decision making, and a building block for reinforcement learning. In a typical contextual bandit formulation, at every time $t$, we are given a state (context) $s_t$ drawn independently from a state distribution, and must decide on an action $a_t$. Then, a reward $r(s_t,a_t)$ is revealed to the learner, and the process repeats. A standard objective is to minimize regret in time horizon $T$ against the best possible reward, knowing the dynamics of the world (reward function and state distribution). For more background on contextual bandits, we refer readers to \cite{foster2020beyond, simchi2020bypassing}.

In the inventory control problem, the time-independent property of contextual bandits is violated by the planning problem, since the best action at time $t$ depends on the policy followed in future time steps. Phrased differently, the trajectory of states (contexts) are dependent on the past actions. Instead, we propose to assume that future actions continue to follow policy $\pi$, which will allow the use of contextual bandits to determine a sequence of actions that yield an equilibrium policy. In the current formulation, this corresponds to evaluating the action at time $t$, given all future actions follow $\pi$. In practice, this would means that in $T$ time, we only generate one observation sample, which is not conducive to online learning. Therefore, we make the following assumption:
\begin{assumption}[Effective horizon]
\label{horizon}
There exists an integer $H \geq 0$ such that, under all permissible policies, the reward at time $t$ only depends on at most the past $H$ states $s_t$ and past $H$ actions $a_t$. Furthermore, given two different actions played at time $t$ with the same state $s_t$, after $H$ steps under the same actions, the two policies will reach the same state $s_{t+H}$.
\end{assumption}
\begin{remark}
In practice, the first part of Assumption \ref{horizon} means that actions played at time $t$ no longer have an effect on the reward more than $H$ steps into the future. Intuitively, after a few inventory order and arrival cycles (which take vendor lead time steps), the inventory purchased in the past will be sold and no longer contribute to costs or reward. The second part of Assumption \ref{horizon} is effectively a coupling assumption, where after a certain amount of time two different policies will reach the same state. However, it is much easier to achieve since we are controlling the future actions to be the same, regardless how the state evolves. This subtle difference means that whenever inventory is fully consumed (which would be expected to happen every so often for most reasonable policies) or forcibly removed, the states will necessarily couple. 
\end{remark}
Under the assumption, we can evaluate the effect of any action $H$ time steps later, and therefore effectively learn via contextual bandits over a much larger time horizon $T$ (notice that when $H=0$ we have the traditional contextual bandit setup). We will proceed with the contextual bandit framework followed by \cite{foster2020beyond}, which is a regret-optimal contextual bandit algorithm that effectively reduces the contextual bandit problem to online regression. We assume the learner has access to a parametric family of functions $\F$ (e.g. linear functions, neural networks), which takes as input the the current state and proposed action, along with the past $H$ states and actions, to predict the expected reward. A critical assumption is that one of functions $f \in\F$ indeed expresses the true reward well, i.e. the realizability assumption: 

\begin{assumption}[Realizability]
\label{realizable}
Assume there is a function $f^* \in \F$ such that \[f^*(s_t,s_{t-1},...,s_{t-H},a_t,a_{t-1},...,s_{t-H}) = \E[r(s_t,a_t)|s_{t-1},...,s_{t-H},a_{t-1},...,s_{t-H}].\]
\end{assumption}


\begin{remark}
Depending on the problem, Assumption \ref{realizable} may feel very strong; however, it is a standard assumption in the contextual bandit literature (\cite{foster2020beyond, simchi2020bypassing}), and required to ensure that we are able to learn properly by predicting rewards from historical data. One can surmise that very complex models such as neural networks would be able to satisfy this assumption, at least approximately. Though we do not discuss approximate realizability in this paper, similar results can be achieved with an added term to account for realizability error (see \cite{foster2020beyond}).
\end{remark}

\subsection{Algorithm}
We now give a brief overview of the main ideas of our contextual bandit algorithm. First, we must specify the set of times $\T$ defined in Definition \ref{def:ep}. By Assumption \ref{horizon}, the reward at any time $t$ depends on at most the past $H$ states and actions, and we effectively reset our problem every $H$ time steps. To ensure we are only potentially changing and evaluating one of these actions (keeping the others fixed to the original policy $\pi$), we can only play one action every $H$ steps. 

Therefore, let $\T$ be any set of times $\tau_0<\tau_1<...<\tau_n$ such that for all $i$, $\tau_{i+1}-\tau_i \geq H$, we refer to $\T$ as the set of target bandit weeks. For every time $t$, the bandit first checks whether the current time is a target bandit week. If not, the original policy $\pi$'s action is followed, otherwise the best model $f_t \in \F$ is fitted using historical data. Model $f_t$ is used to generate action probabilities for every action $a \in \A$, and the action played is sampled according to the action probability distribution. Methods for assigning action probabilities may include, but are not not limited to, epsilon-greedy selection or inverse gap weighting \cite{foster2020beyond}. The exploration parameter $\rho$ is  provided to the learner to account for the corresponding parameters for various exploration policies. The full details can be found in Algorithm \ref{alg}.

\begin{algorithm}[H]
\label{alg}
\SetAlgoLined
{\bf Initialize:} time horizon $T$, effective planning horizon $H$, exploration parameter $\rho$, function family $\F$, target bandit weeks $\T$, default policy $\pi$\;
 \For{$t = 1,2,...,T$}{
 
 \If{$t \notin \T$}{
      play action $a_t$ according to policy $\pi$;
   }\ElseIf{$t \in \T$}{
                         
    Determine best fit $f_t \in \F$ using historical data \;
  Observe $x_t$\;
  \For{$a \in \A$}{
   Compute $f_t(x_t,a)$\;
   Assign action probabilities $p_t(a)$ based on exploration policy and $\rho$ (e.g. epsilon-greedy, inverse gap weighting)\;
   }
  Play action $a_t \sim p_t$\;
  Update exploration parameter $\rho$ if necessary\;
    }                           
 }
 \caption{Equilibrium Policy via Contextual Bandits}
\end{algorithm}

We note that Algorithm \ref{alg} prescribes a method for training and executing an online contextual bandit such that the resulting policy is an equilibrium policy at times $t \in \T$ (see Theorem \ref{thm}). This gives a method that can detect whether a policy is in equilibrium (if the bandit is applied and results in many action changes, the policy is likely not in equilibrium). However, we can also check whether we are in equilibrium at \emph{any} time. To do so, we use the most-recently learned model $f$, and check whether our actions are close to the best action implied by $f$ (but we do not actually execute the bandit action unless $t\in\T$). 

More precisely, let $\F':= \{f_{\tau_0},f_{\tau_0},...,f_{\tau_n}\}$ be the set of models selected by our bandit for every time $t \in \T$, and let $f'_t$ be the most recent model $f \in \F'$ fit before time $t$. Then, to evaluate the action under $\pi$ at time $t$, we compare the action induced by the original policy $\pi$ to the $\epsilon$-equilibrium action $\argmax_a f'_t(x_t,a)$. As $t$ increases, $f'_t$ is closer to $f^*$ and thus $\epsilon$ is smaller. For very large $t$, $f'_t$ will give us very close to the true equilibrium action.

\subsection{Main Result}
In this section we show that the contextual bandit algorithm (Algorithm \ref{alg}) results in an equilibirum policy at every time $t\in \T$.

\begin{theorem}
\label{thm}
Assume Assumptions 1 and 2 hold.  Let $\pi'$ be the bandit policy in Algorithm \ref{alg} and $\pi$ the original policy being evaluated. 
Define 
\[   
\tilde\pi = 
     \begin{cases}
       \pi &\quad\text{if } t \notin \tau\\
       \pi' &\quad\text{otherwise.} \\ 
     \end{cases}
\]
Then, for some $\epsilon >0$, the policy $\tilde\pi$ is an $\epsilon$-equilibrium policy at every time $t \in \tau$. Furthermore, as $|\T| \rightarrow \infty$ the learned policy asymptotically is the true best action for each time $t \in \tau$.
\end{theorem}

\begin{proof}
Let $\pi^*$ be the policy that picks the hindsight best action at every time $t \in \T$, and $\tilde\pi$ be as defined previously. Then, the corresponding regret:
\[Regret(T) = \sum_{t=1}^T \gamma^t r^{\pi^*}(s_t,a_t) -  \gamma^tr^{\tilde\pi}(s_t,a_t). \]
By Assumption \ref{horizon}, this can be rewritten as:
\[\sum_{i=0}^{n-1} \sum_{t= \tau_i}^{\tau_{i+1}-1} \gamma^tr^{\pi^*}(s_t,a_t) -  \gamma^tr^{\tilde\pi}(s_t,a_t) = \sum_{i=0}^{n-1} \sum_{t= \tau_i}^{\tau_i + H} \gamma^tr^{\pi^*}(s_t,a_t) -  \gamma^tr^{\tilde\pi}(s_t,a_t).\]
The inner summation 
\[\sum_{t= \tau_i}^{\tau_i + H} \gamma^tr^{\pi^*}(s_t,a_t) -  \gamma^tr^{\tilde\pi}(s_t,a_t)\]
is precisely one time step of a well-posed contextual bandit problem, with the decision being made at every time $t\in\T$. The reward for playing action $a_t$ at state $s_t$ (for $t\in\T$) is \[f_t(s_t,a_t) := \sum_{t= \tau_i}^{\tau_i + H} \gamma^tr^{\tilde\pi}(s_t,a_t).\]
By Assumption \ref{realizable}, the problem is realizable and hence any standard contextual bandit algorithm (e.g., epsilon-greedy or \cite{foster2020beyond}) achieves sub-linear regret, so that for any $n$, there exists an $\epsilon > 0$ such that 
\[\sum_{i=0}^{n-1} \sum_{t= \tau_i}^{\tau_i + H} \gamma^tr^{\pi^*}(s_t,a_t) -  \gamma^tr^{\tilde\pi}(s_t,a_t) \leq \epsilon n \leq \epsilon T.\]

Recall $R^{\pi}_T:= \sum_{t=1}^T \gamma^t r^{\pi}(s_t,a_t)$, and define $\Pi_{\T}$ to be the set of policies such that for any $\pi' \in \Pi_{\T}$ (which induces states $s'_1,s'_2,...,s'_T$), $\pi(s_t) = \pi'(s'_t)$ for all times $t \not\in \T$ (additionally, $s_t = s'_t$ by Assumption \ref{horizon}). By construction, $\pi^* \in \Pi_{\T}$, and thus by optimality, $\max_{\pi' \in \Pi}R^{\pi'}_{T} = R^{\pi^*}_{T}$. Then, for all $t\in\T$, we have

\[\max_{\pi' \in \Pi_{\T}} R^{\pi'}_{T} - R^{\tilde\pi}_{T} = R^{\pi^*}_{T} - R^{\tilde\pi}_{T} \leq  \epsilon T.\]
The definition of equilibrium policy follows after taking expectation and averaging over all time steps $T$. Furthermore, as $|\T| \rightarrow \infty$, because our contextual bandit algorithm achieves sub-linear regret, realizability implies that $f_t \rightarrow f^*$ necessarily.

\end{proof}

Contextual bandits gives us a way to predict the best action, which can be used to verify the conditions of $\epsilon$-equilibrium (relating $\epsilon$ to the level of approximation error of the contextual bandit algorithm is straightforward with additional standard assumptions such as Lipschitz-continuity of the actions).

\section{Empirical Study}
We test our approach in a simulated inventory management environment \cite{scotrl}. We first train a model by randomly selecting actions in the training data, and test this model on out-of-sample test data. Our model here is effectively a warm-started version of Algorithm \ref{alg} evaluated on one decision period, that can be repeatedly updated online with additional data.

In our scenario, we evaluate various inventory control policies, in particular one based on a traditional newsvendor model and one based on reinforcement learning (RL). However, \emph{a priori} we have observational evidence that these two policies perform fairly well in practice, and are not easy to beat with minor adjustments. Therefore, we would also include two modified policies: one where the newsvendor total inventory positions are all multiplied by $2$, and one where they are all multiplied by $0.5$.

When evaluating each policy, we assume that $H=12$. For each product, we use static (price, cost, etc.) and dynamic features (previous states and actions) as part of the state $x_t$. In order to vary the weeks the bandit acts on, each product has a different, randomly selected week in the test data where the bandit decides the action. The bandit is allowed to choose between three actions: multiplying the original action by $\{0.8, 1,1.2\}$ (corresponding to \{down, same, up\}). Once the bandit action is selected, the remainder of the time horizon is rolled out according to the original policy's action. 

The model class $\F$ we consider is a fairly generic class of linear regression-based models (with an added quadratic term to account for non-linearity in the action) parameterized by $\theta := (\theta_1,\theta_2,\theta_3)$. For $a_t \in \{0.8,1,1.2\}$ let
\[ f(x_t, a_t) = \theta_1^Tx_t+(a_t-1)\theta_2^Tx_t+(a_t-1)^2\theta_3^Tx_t.\]

The bandit will evaluate each potential action using the learned model $f$ and current state $x_t$, and pick the best action based on the predicted values. However, in order to account for model inaccuracy, we include a tunable hyperparameter $\epsilon$ such that we boost the predicted reward of the baseline action by a factor of $1+\epsilon$. This forces the bandit to pick the original baseline action when predicted values are very close.

We measure two outcomes: the distribution of \{down, same, up\} actions recommended by the bandit, as well as the improvement in actual reward. Because the action change (of at most $20\%$) is applied only once in our test data, the actual change in reward is small and insignificant (always less than $1\%$). Therefore, we instead observe the improvement $\Delta\I$ of the bandit's reward compared the original baseline policy (which always picks $a = 1$) and the oracle policy that knows and plays the optimal action. Formally, if $R(x)$ is the total reward of policy $x \in \{bandit, baseline, oracle\}$, then our improvement metric is defined as
\[\Delta\I := \frac{R(bandit) - R(baseline)}{R(oracle) - R(baseline)}. \]

Our results are displayed in Table \ref{tab}, where the bandit is applied to newsvendor and RL policies, as well as a modified version of newsvendor where the total inventory position is adjusted by a factor of $2$, up and down.

\begin{table}[h]
\label{tab}
\centering
\begin{tabular}{|c |c |c |c|} 

 \hline
 Policy & Percent actions: \{down, same, up\} & $\Delta\I$ \\ [0.5ex] 
 \hline \hline
 newsvendor +  bandit &  \{0.07, 99, 0.03\}  & +0.12 \\ 
 \hline

 newsvendor/2 +  bandit &  \{0.1, 4.1, 95.8\}  & +91.7 \\ 
 \hline

  newsvendor*2 +  bandit & \{17.7, 81.2, 1.1\}  & +39.7 \\ 
 \hline

 RL +  bandit &  \{2.2, 97.7, 0.1\} & +4.4 \\ 
 \hline

\end{tabular}
\caption{Applying contextual bandit to evaluate policies in a simulated inventory buying environment.}

\end{table}

The results in Table \ref{tab} are consistent with our original hypothesis that, for the most part, the original newsvendor and RL policies are not easily beatable. Our experiment validates this by showing that these policies are equilibrium policies, so meaningful improvements to the policy require more than just occasional changes to the policy. The results also show that when the total inventory positions of the newsvendor policy are scaled, the bandit reacts accordingly to detect how to improve the policy. Interestingly, being under-stocked (newsvendor$/2$) appears to deviate much more significantly from equilibrium compared to being over-stocked (newsvendor$*2$), likely due to the asymmetric nature of the reward function.

\section{Conclusion}
Truly optimal inventory control problems are typically intractable in most practical settings due to stochasticity and other unknown dynamics of the world. Since it is usually infeasible to counterfactually determine the best policy in hindsight, in order to design good inventory control policies, we must also be able to evaluate them effectively. We introduced the equilibrium property to serve as an achievable benchmark for inventory control policies. Though not all equilibrium policies are optimal, they do have the attractive property that they cannot be improved by changing a small fraction of actions. In practice, this could mean that it is difficult to show that the policy is obviously not a good policy. In this paper, we presented a formal mathematical notion of equilibrium policies, and showed that contextual bandit algorithms can be used to achieve such policies. 

We note that our experiments are trained and tested on a simulated environment, which is also the same environment used to train policies such as the RL policy. As a result, it is of no surprise the bandit finds that the RL policy is close to equilibrium. One problem with training a simulator-based model is that oftentimes, the real world and the simulated world diverge due to inaccurate modeling or fundamental regime shifts. In order to account for new real-world observations, we would have to retrain the simulator and the RL policy after every new data point, which would become computationally expensive. However, since contextual bandits can learn both efficiently and effectively online, we would like to study whether techniques such that those presented in this paper can serve as a stepping stone to addressing this issue. For future work, we will attempt to study this problem and try to answer the question: can contextual bandits efficiently improve, in an online fashion, an inventory policy trained under a simulated environment to behave in equilibrium with respect to the real world? 

\bibliographystyle{plain}
\bibliography{references}

\end{document}